\documentclass[twoside,11pt]{article}

\usepackage{jmlr2e}

\usepackage[utf8]{inputenc} 
\usepackage[T1]{fontenc}    
\usepackage{lmodern}

\usepackage[normalem]{ulem} 
\usepackage{url}            
\usepackage{booktabs}       
\usepackage{amsfonts}       
\usepackage{nicefrac}       
\usepackage{microtype}      
\usepackage{graphicx}
\usepackage{subfigure} 
\usepackage{tikz}
\usepackage{amsmath,amsfonts,amssymb,dsfont}
\usepackage{stmaryrd}
\usepackage{enumitem}
\usepackage{wrapfig}
\def\+{\!+\!}

\newcommand\il[1]{\langle #1 \rangle}

\def\CW{\mathrm{cw}}

\def\Z{\mathbb{Z}}
\def\R{\mathbb{R}}

\def\erf{\mathrm{erf}}

\def\D{\mathcal{D}}
\def\E{\mathcal{E}}
\def\Z{\mathcal{Z}}

\def\X{\mathbf{X}}
\def\Y{\mathbf{Y}}

\newcommand*\dif{\mathop{}\!{d}}

\def\1F1{\mbox{$_{1}{F}_{\!1}$}}

\def\reg{\mathrm{sm}}

\usepackage{color}

\usepackage{todonotes}

\newcommand{\eref}[1]{(\ref{#1})}

\jmlrheading{1}{2000}{1-48}{4/00}{10/00}{meila00a}{Marina Meil\u{a} and Michael I. Jordan}
\jmlrheading{1}{2019}{x-xx}{x/xx}{xx/xx}{taboretal}{Knop and Tabor and Spurek and Podolak and Mazur and Jastrzębski}

\ShortHeadings{Learning with Mixtures of Trees}{Meil\u{a} and Jordan}
\ShortHeadings{Cramer-Wold AutoEncoder}{Knop and Tabor and Spurek and Podolak and Mazur and Jastrzębski}
\firstpageno{1}

\usepackage[mathlines]{lineno}
\begin{document}
\title{Cramer-Wold AutoEncoder}

\author{\name Szymon Knop \email szymonknop@gmail.com\\
       \name Jacek Tabor \email jacek.tabor@uj.edu.pl \\
       \name Przemysław Spurek \email przemyslaw.spurek@uj.edu.pl \\ 
       \name Igor Podolak \email igor.podolak@uj.edu.pl \\ 
       \name Marcin Mazur \email marcin.mazur@uj.edu.pl \\
       \addr Faculty of Mathematics and Computer Science \\ Jagiellonian University, Krak\'ow, Poland 
       \AND       
       \name Stanisław Jastrzębski \email staszek.jastrzebski@gmail.com \\
       \addr Department of Radiology\\
New York University School of Medicine,
        New York, United States
       }

\editor{Kevin Murphy and Bernhard Sch{\"o}lkopf}

\maketitle

\begin{abstract}
Computing the distance between the true and the sample distributions is a key component of most state-of-the-art generative models.  Inspired by prior work on Sliced-Wasserstein Autoencoders (SWAE) and WAE-MMD we construct a new generative model -- a Cramer-Wold AutoEncoder (CWAE). A fundamental component of CWAE is the characteristic kernel we construct, which we call Cramer-Wold kernel. Its main distinguishing feature is that it has a closed-form of the kernel product of radial gaussians.
Consequently, CWAE model has a closed-form for the distance between the posterior and the Normal prior, which simplifies the optimization procedure by removing the need to sample to compute the loss function. At the same time, CWAE performance often improves upon WAE-MMD and SWAE on standard benchmarks.
\end{abstract}


\section{Introduction}


One of the crucial aspects in construction of generative models is devising efficient methods for computing and minimizing distance between the true and the model prior distribution. Originally in Variational Autencoder (VAE)~\citep{kingma2014autoencodingvariationalbayes} this computation was carried out using the variational scheme. A significant improvement was brought by the introduction of Wasserstein metric~\citep{tolstikhin2017wasserstein} in the construction of WAE-GAN and WAE-MMD models, which relaxed the need for variational methods.  WAE-GAN requires a separate optimization problem to be solved to approximate the used divergence measure, while in WAE-MMD the discriminator has the closed form obtained from a characteristic kernel\footnote{Characteristic kernel is a kernel that is injective on distributions, see e.g.~\citet{muandet2017kernel}.}.

\begin{figure}
\centering
\includegraphics[height=6cm]{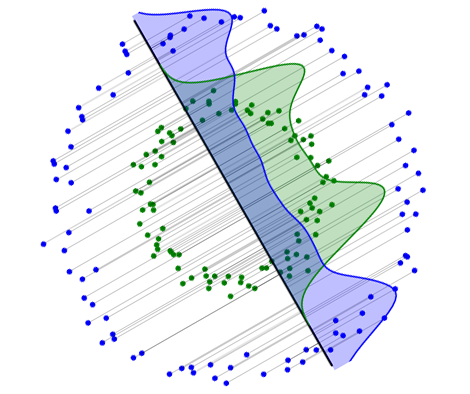} 
\caption{Cramer-Wold distance between two sets is obtained as the mean squared $L_2$-distance of their smoothed projections on all one-dimensional lines. Figure shows an exemplary (one of many) projection.}
\label{fig:slice}
\end{figure}

Most recently~\citet{kolouri2018sliced} introduced the Sliced-Wasserstein Autoencoder (SWAE), which simplifies this distance computation even further. The main innovation of SWAE was the introduction of the sliced-Wasserstein distance, a fast to estimate metric for comparing two distributions based on the mean Wasserstein distance of one-dimensional projections. However, even in SWAE there is no closed-form analytic formula that would enable computing the distance of the sample from the standard normal distribution. Consequently, in SWAE two types of sampling are needed: (i) sampling from the prior distribution and (ii) sampling over one-dimensional projections. 


Our main contribution is the introduction of Cramer-Wold kernel between distributions and based on it Cramer-Wold AutoEncoder model (CWAE). Crucially, the Cramer-Wold kernel is a characteristic kernel which has a closed-form for the product of radial Gaussians, see Eq.~\eref{eq:formula}. We use it to construct CWAE model in which the cost function has a closed analytic formula. Thus CWAE can be interpreted as a model retaining the best qualities of both SWAE and WAE-MMD: we use a characteristic kernel to discriminate distributions as in WAE-MMD, but its formula comes from using one-dimensional projections as in SWAE. We demonstrate on standard benchmarks that CWAE is faster to optimize and retains or even improves performance compared to both WAE-MMD and SWAE. 

In sections~\ref{sec:cw_distance} and~\ref{sec:compute-phi}, we introduce and theoretically analyse the Cramer-Wold distance, and follow with the formal definition of a Cramer-Wold kernel in Section~\ref{sec:cw-kernel}. Readers interested mainly in the construction of CWAE can proceed directly to Section~\ref{se:CWAE}. Section~\ref{se:ex} contains experiments. Finally, we conclude in Section~\ref{se:con}.


\section{Related work}
Generative models, that were to be easily trained, scalable, with real generative powers, is a long-time challenge of Machine Learning. First generative models, the Boltzmann Machine, Deep Belief Networks and Deep Boltzmann Machines were trained with a Monte Carlo approach~\citep{hinton1995wake,HintonOsinderoTeh2006FastLearning,salakhutdinovhinton2009deep}. The Monte Carlo MCMC training, slow and imprecise, gave way to the variational method, able to learn using a direct gradient, together with the advent of generative auto-encoder models. First came the variational autoencoder VAE~\citep{kingma2014autoencodingvariationalbayes} to use a stochastic encoder and a variational evidence lower bound ELBO. Variational inference is faster and scales much better with large data than MCMC.

A new paradigm of adversarial training was introduced with the advent of new paradigm of adversarial training and the
Generative Adversarial Networks GAN~\citep{goodfellow2014generative}. A discriminator network trained to distinguish between true and generated examples, pushes a generator to produce even better mappings.

\cite{makhzani2015adversarial} introduced the adversarial autoencoder. The adversarial approach was used not to check the generated data distribution compatibility with data distribution $p_{DATA}$, but to control the latent space distribution.

The Wasserstein autoencoders changed the method to check the compliance of the latent distribution with the prior one selected~\citep{tolstikhin2017wasserstein,arjovsky2017towards,arjovsky2017wassersteingan}. This gave way to two solutions: a model with adversarial training of prior compliance, an adversarial one, in case of square cost used coinciding with the adversarial autoencoder, and a maximum mean discrepancy MMD approach~\citep{Gretton2012kerneltwosampletest}.

An important ingredient of WAE-MMD model is given by the kernel used. For the review of kernels we refer the reader to \cite{muandet2017kernel}. Typically, one uses characteristic kernels \citep{sriperumbudur2011universality}, i.e. kernels which define a metric on the space of distributions. Moreover, the preference is given to kernels which do not decrease to zero exponentially, see \cite{tolstikhin2017wasserstein}, as it can cause problems in the minimization procedure due to small gradient.

\section{Cramer-Wold distance}\label{sec:cw_distance}

Motivated by the prevalent use of normal distribution as prior in modern generative models, we investigate whether it is possible to simplify and speed up the optimization of such models. As the first step towards this, we introduce Cramer-Wold distance, which has a simple analytical formula for computing normality of high-dimensional samples. On a high level our approach uses the traditional $L_2$ distance of kernel-based density estimation, computed across multiple single-dimensional projections of the true data and the output distribution of the model. We base our construction on the following two popular tricks of the trade: sliced based decomposition and smoothing of distributions.

\paragraph{Sliced-based decomposition of a distribution} Following the footsteps of~\citet{kolouri2018sliced, deshpande2018generative}, the first idea is to leverage the Cramer-Wold Theorem~\citep{cramer1936some} and Radon Transform~\citep{deans1983radon} to reduce computing distance between two distributions to one-dimensional calculations. For $v$ in the unit sphere $S_D \subset \R^D$, the projection of the set $X \subset \R^D$ onto the space spanned by $v$ is given by $v^TX$ and the projection of $N(m,\alpha I)$ is $N(v^Tm,\alpha)$. Cramer-Wold theorem states that two multivariate distributions can be uniquely identified by their all one-dimensional projections. For example, to obtain the key component of SWAE model, i.e.~the sliced-Wasserstein distance between two samples $X,Y \in \R^D$, we compute the mean Wasserstein distance between all one-dimensional projections:
\begin{equation} \label{eq:swd}
d_W(X,Y)=\int_{S_D} d_W(v^TX,v^TY) \, d\sigma_D(v),
\end{equation}
where $S_D$ denotes the unit sphere in $\R^D$ and $\sigma_D$ is the normalized surface measure on $S_D$. This approach is effective since the one-dimensional Wasserstein distance between samples has the closed-form, and therefore to estimate \eref{eq:swd} one has to sample only over the projections.

\paragraph{Smoothing distributions} Using the sliced-based decomposition requires us to define distance between two sets of samples, in a single dimensional space. To this end, we will use a trick-of-trade applied commonly in statistics in order to compare samples or distributions which is to first smoothen (sample) distribution with a  Gaussian kernel. For the sample $R=(r_i)_{i=1..n} \subset \R$ by its smoothing with Gaussian kernel $N(0,\gamma)$ we understand
\begin{linenomath*}
\begin{equation*}
\reg_\gamma(R)=\frac{1}{n}\sum_i N(r_i,\gamma), 
\end{equation*}
\end{linenomath*}
where by $N(m,S)$ we denote the one-dimensional normal density with mean $m$ and variance~$S$. This produces a distribution with regular density, and is commonly used in kernel density estimation. If $R$ comes from the normal distribution with standard deviation close to one, the asymptotically optimal choice of $\gamma$ is given by the Silverman's rule of thumb $\gamma=(\tfrac{4}{3n})^{2/5}$, see~\citet{silverman1986density}. For a continuous density $f$, its smoothing $\reg_\gamma(f)$ is given by the convolution with $N(0,\gamma)$, and in the special case of Gaussians we have $\reg_\gamma(N(m,S))=N(m,S+\gamma)$. While in general kernel density estimations works well only in low-dimensional spaces, this fits the bill for us, as we will only compute distances on single dimensional projections of the data. 

\paragraph{Cramer-Wold distance} We are now ready to introduce the \emph{Cramer-Wold distance}. In a nutshell, we propose to compute the squared distance between two samples by considering the mean squared $L_2$ distance between their smoothed projections over all single dimensional subspaces. By the squared $L_2$ distance between functions $f,g:\R \to \R$ we refer to $\|f-g\|_2^2=\int |f(x)-g(x)|^2 dx$. A key feature of this distance is that it permits a closed-form in the case of normal distribution.

The following algorithm fully defines the Cramer-Wold distance between two samples $X=(x_i)_{i=1..n}, Y =(y_j)_{j=1..k} \subset \R^D$ (for illustration of Steps 1 and 2 see Figure~\ref{fig:slice}):
\begin{enumerate}
\item given $v$ in the unit sphere $S(0,1) \subset \R^D$ consider the projections $v^T X=(v^T x_i)_{i=1..n}$ and $v^T Y=(v^T y_j)_{j=1..k}$\,,
\item compute the squared $L_2$ distance of the densities $\reg_\gamma(v^TX)$ and $\reg_\gamma(v^TX)$:
\begin{linenomath*}
\begin{equation*}
\|\reg_\gamma(v^TX)-\reg_\gamma(v^TY)\|_2^2\,,
\end{equation*}
\end{linenomath*}
\item to obtain squared Cramer-Wold distance average
the above formula over all possible $v \in S_D$.
\end{enumerate}


The key theoretical outcome of this paper is that the computation of the Cramer-Wold distance can be simplified to a closed-form solution. Consequently, to compute the distance of two samples there is no need of finding the optimal transport like in WAE or the necessity to sample over the projections like in SWAE. 

\begin{theorem} \label{th:21}
Let $X=(x_i)_{i=1..n}$, $Y=(y_j)_{j=1..n} \subset \R^D$ be given\footnote{For clarity of presentation we provide here the formula for the case of samples of equal size.}.
We formally define the squared Cramer-Wold distance using the formula
\begin{linenomath*}
\begin{equation*}
d^2_{\CW}(X,Y):=\int_{S_D}\|\reg_\gamma(v^TX)-\reg_\gamma(v^TY)\|^2_{2} \dif \sigma_D(v).
\end{equation*}
\end{linenomath*}
Then
\begin{align}\label{eq:th31}
\begin{split}
d^2_{\CW}(&X,Y)=\tfrac{1}{2n^2\sqrt{\pi \gamma}}
\bigg(\sum \limits_{ii'}\phi_D\big(\tfrac{\|x_i-x_{i'}\|^2}{4\gamma}\big)
+\sum \limits_{jj'}\phi_D\big(\tfrac{\|y_j-y_{j'}\|^2}{4\gamma}\big)-2\sum \limits_{ij} \phi_D\big(\tfrac{\|x_i-y_j\|^2}{4\gamma}\big) \bigg),
\end{split}
\end{align}
where $\phi_D(s)=\1F1(\tfrac{1}{2};\tfrac{D}{2};-s)$ and $\1F1$ is the Kummer's confluent hypergeometric function (see, e.g.,~\citet{Barnard1998}). Moreover, $\phi_D(s)$ has the following asymptotic formula valid for $D\geq 20$:
\begin{equation} \label{eq:pi}
\phi_D(s)\approx (1+\tfrac{4s}{2D-3})^{-1/2}.
\end{equation}

\end{theorem}

To prove the Theorem~\ref{th:21} we shall need the following crucial technical proposition.

\begin{proposition} \label{pr:31}
Let $z \in \R^D$ and $\gamma >0$ be given. Then 
\begin{equation} \label{eq:fr}
\int \limits_{S_{D}} N(v^Tz,\gamma)(0)\dif\sigma_D(v)=
\frac{1}{\sqrt{2\pi \gamma}}\phi_D\left(\frac{\|z\|^2}{2\gamma}\right).
\end{equation}
\end{proposition}

\begin{proof}
By applying an orthonormal change of coordinates, without loss of generality, we may assume that $z=(z_1,0,\ldots,0)$, and then $v^Tz=z_1v_1$ for $v=(v_1,\ldots,v_D)$. Consequently we get
\begin{linenomath*}
\begin{equation*}
\begin{array}{c}
\int\limits_{S_{D}} N(v^Tz,\gamma)(0)\dif\sigma_D(v)=\int\limits_{S_{D}} N(z_1v_1,\gamma)(0) \dif\sigma_D(v).
\end{array}
\end{equation*}
\end{linenomath*}
Making use of the formula for slice integration of functions on spheres~\citep[Corollary~A.6]{axler2013harmonic} we get:
\begin{linenomath*}
\begin{equation*}
\begin{array}{r@{}l}
\int\limits_{S_D} f \dif\sigma_D =&
\tfrac{V_{D-1}}{V_D}
\int\limits_{-1}^1 (1-x^2)^{(D-3)/2} \cdot \int\limits_{S_{D-1}} 
f(x,\sqrt{1-x^2}\, \zeta ) \dif\sigma_{D-1}(\zeta) \dif x,
\end{array}
\end{equation*}
\end{linenomath*}
where $V_K$ denotes the surface volume of a sphere $S_K \subset \R^K$. 
Applying the above equality for the function $f(v_1,\ldots,v_D)=N(z_1 v_1,\gamma)(0)$ and 
$s=z_1^2/(2\gamma)=\|z\|^2/(2\gamma)$ we consequently get
that the LHS of \eref{eq:fr} simplifies to
\begin{linenomath*}
\begin{equation*}
\begin{array}{l}
\frac{V_{D-1}}{V_D}
\frac{1}{\sqrt{2\pi \gamma}}\int_{-1}^1  (1-x^2)^{(D-3)/2}\exp(-s x^2) \dif x,
\end{array}
\end{equation*}
\end{linenomath*}
which completes the proof since $V_{K}=\frac{2 \cdot \pi ^{\frac {K}{2}}}{\Gamma \left({\frac {K}{2}}\right)}$ and
\begin{linenomath*}
\begin{equation*}
\begin{array}{r@{}l}
\int_{-1}^1 \exp(-sx^2)
(1-x^2)&^{(D-3)/2} \dif x
=\sqrt{\pi } \frac{\Gamma \left(\frac{D-1}{2}\right)}{\Gamma(\frac{D}{2})} \, \1F1\left(\frac{1}{2};\frac{D}{2};-s\right)
\end{array}
\end{equation*}
\end{linenomath*}
\end{proof}

\begin{proof}[Proof of Theorem~\ref{th:21}] Directly from the definition of smoothing we obtain that
\begin{equation} \label{eq:si}
\begin{array}{r@{}l}
d^2_\CW(X,Y)= &
\int_{S_D}\big\|\tfrac{1}{n}\sum_i N(v^Tx_i,\gamma) - \tfrac{1}{n}\sum_j N(v^Ty_j,\gamma)\big\|^2_2 \dif \sigma_D(v).
\end{array}
\end{equation}
Applying now the one-dimensional formula for the $L_2$-scalar product of two Gaussians:
\begin{linenomath*}
\begin{equation*}
\il{N(r_1,\gamma_1),N(r_2,\gamma_2)}_2=N(r_1-r_2,\gamma_1+\gamma_2)(0)
\end{equation*}
\end{linenomath*}
and the equality $\|f-g\|^2_2=\il{f,f}_2+\il{g,g}_2-2\il{f,g}_2$ (where $\il{f,g}_2=\int f(x)g(x) dx$), we simplify the squared-$L_2$ norm in the integral of RHS of \eref{eq:si} to
\begin{linenomath*}
\begin{equation*}
\begin{array}{r@{}l}
\|\tfrac{1}{n}\sum \limits_i N(v^Tx_i,\gamma)-\tfrac{1}{n}\sum \limits_i N(v^Ty_j,\gamma)\big\|^2_2
=&\tfrac{1}{n^2}\il{\sum \limits_i N(v^Tx_i,\gamma),\sum \limits_{i} N(v^Tx_{i},\gamma)}_2\\
&\;+\tfrac{1}{n^2}\il{\sum \limits_j N(v^Ty_j,\gamma),\sum \limits_{j} N(v^Ty_{j},\gamma)}_2\\
&\;{-\tfrac{2}{n^2}\il{\sum \limits_i N(v^Tx_i,\gamma),\sum \limits_j N(v^Ty_j,\gamma)}_2}\\
=&\frac{1}{n^2}\sum \limits_{ii'} N(v^T(x_i-x_{i'}),2\gamma)(0)\\
&\;+\tfrac{1}{n^2}\sum \limits_{jj'} N(v^T(y_j-y_{j'}),2\gamma)(0)\\
&\;-\tfrac{2}{n^2}\sum \limits_{ij} N(v^T(x_i-y_j),2\gamma)(0).
\end{array}
\end{equation*}
\end{linenomath*}

Applying proposition~\ref{pr:31} directly we obtain formula \eref{eq:th31}. Proof of the formula for the asymptotics of the function $\phi_D$ is provided in in next section.
\end{proof}

Thus to estimate the distance of a given sample $X$ to some prior distribution $f$, one can follow the common approach and take the distance between $X$ and a sample from $f$. As~the main theoretical result of the paper we consider the following theorem, which states that in the case of standard Gaussian multivariate prior, we can completely reduce the need for sampling (we omit the proof since it is similar to that of Theorem~\ref{th:21}).

\begin{theorem}
Let $X=(x_i)_{i=1..n} \subset \R^D$ be a given sample. We formally define
\begin{linenomath*}
\begin{equation*}
\begin{split}
d^2_{\CW}&(X,N(0,I)):=\int_{S_D} \|\reg_\gamma(v^TX)-\reg_\gamma(N(0,1))\|^2_2 d\sigma_D(v).
\end{split}
\end{equation*}
\end{linenomath*}
Then
\begin{linenomath*}
\begin{equation*}\label{eq:disc}
\begin{array}{l}
d^2_{\CW}(X,N(0,I)) 
= \tfrac{1}{2n^2\sqrt{\pi}}
\big(\tfrac{1}{\sqrt{\gamma}}\sum_{i,j} 
\phi_D(\tfrac{\|x_i-x_j\|^2}{4\gamma}) +\tfrac{n^2}{\sqrt{1+\gamma}}-\tfrac{2n}{\sqrt{\gamma+\frac{1}{2}}} \sum_{i} \phi_D(\tfrac{\|x_i\|^2
}{2+4\gamma}) \big).
\end{array}
\end{equation*}
\end{linenomath*}
\end{theorem}

\section{Computation of \texorpdfstring{$\phi_D$}{}}\label{sec:compute-phi} \label{ap:A}

As it was shown in the previous section, the crucial role in the Cramer-Wold distance is given by the function
\begin{linenomath*}
\begin{equation*}
\phi_D(s)=\1F1(\tfrac{1}{2};\tfrac{D}{2};-s) \text{ for } s\geq 0.
\end{equation*}
\end{linenomath*}
Consequently, in this section we focus on the derivation of its basic properties.
First we will provide its approximate asymptotic formula valid for dimensions $D \geq 20$, and then we shall consider the special case of $D=2$ (see Figure~\ref{fig:mnista1}), where we provide the exact formua. 

To do so, let us first recall~\citep[Chapter~13]{abramowitz1964handbook} that the Kummer's confluent hypergeometric function $\1F1$ (denoted also by $M$) has the following integral representation
\begin{linenomath*}
\begin{equation*}
\1F1(a,b,z)=\frac{\Gamma (b)}{\Gamma (a)\Gamma (b-a)}
\int_{0}^{1} e^{{z u}}u^{{a-1}}(1-u)^{{b-a-1}}\,du,
\end{equation*}
\end{linenomath*}
valid for $a,b>0$ such that $b>a$. Since we consider that latent is at least of dimension $D\geq 2$, it follows that 
\begin{linenomath*}
\begin{equation*}
\phi_D(s)=\frac{\Gamma (\tfrac{D}{2})}{\Gamma (\tfrac{1}{2})\Gamma (\tfrac{D}{2}-\tfrac{1}{2})}
\int_{0}^{1} e^{-s u}u^{-1/2}(1-u)^{D/2-3/2}\,du.
\end{equation*}
\end{linenomath*}
By making a substitution $u=x^2$, $du=2xdx$, we consequently get
\begin{align}
\begin{split}
\phi_D(s)&=\tfrac{2\Gamma (D/2)}{\Gamma (1/2)\Gamma (D/2-1/2)}
\int_{0}^{1} \!\!\!\! e^{-s x^2}(1-x^2)^{(D-3)/2}\,dx\\
&=\tfrac{\Gamma (D/2)}{\Gamma (1/2)\Gamma (D/2-1/2)}
\int_{-1}^{1} \!\!\!\! e^{-s x^2}(1-x^2)^{(D-3)/2}\,dx.\label{eq:np}
\end{split}
\end{align}

\begin{figure*}[htb]
\normalsize
\begin{center}
\includegraphics[width=0.28\textwidth]{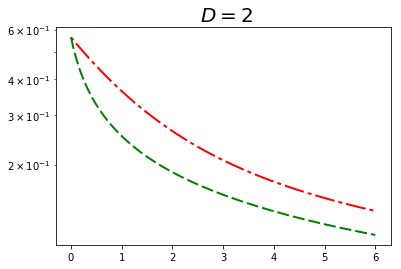}
\includegraphics[width=0.28\textwidth]{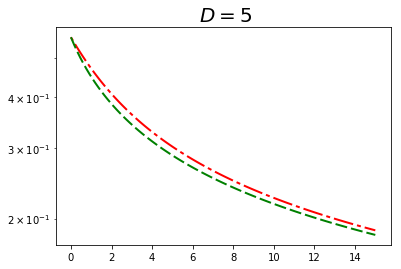}
\includegraphics[width=0.28\textwidth]{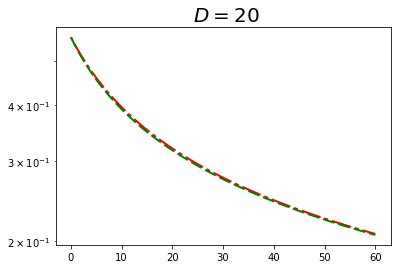}
\end{center}
\caption{Comparison of $\phi_D$ value (red line) with the approximation given by \eref{eq:Phi} (green line) in the case of dimensions $D=2,5,20$.
Observe that for $D=20$, the functions practically coincide.}
\label{fig:mnista1}
\end{figure*}

\begin{proposition} \label{pr:1}
For large\footnote{In practice we can take $D\geq 20$.} $D$ we have
\begin{equation} \label{eq:Phi}
\phi_D(s) \approx
(1+\tfrac{4s}{2D-3})^{-1/2}
\text{ for all }s \geq 0.
\end{equation}
\end{proposition}

\begin{proof}
By \eref{eq:np} we have to estimate asymptotics of
\begin{linenomath*}
\begin{equation*}
\begin{array}{l}
\phi_D(s)=\frac{\Gamma (\tfrac{D}{2})}{\Gamma (\tfrac{1}{2})\Gamma (\tfrac{D}{2}-\tfrac{1}{2})}
\int_{-1}^{1} e^{-sx^2}(1-x^2)^{(D-3)/2}\,dx.
\end{array}
\end{equation*}
\end{linenomath*}
Since for large $D$, for all $x \in [-1,1]$ we have
\begin{linenomath*}
\begin{equation*}
\begin{array}{r@{}l}
(1-x^2)^{(D-3)/2} e^{-sx^2} &\approx(1-x^2)^{(D-3)/2} \cdot (1-x^2)^s =(1-x^2)^{s+(D-3)/2}, 
\end{array}
\end{equation*}
\end{linenomath*}
we get
\begin{linenomath*}
\begin{equation*}
\begin{array}{r@{}l}
\phi_D(s) &\approx 
\frac{\Gamma(\tfrac{D}{2})}{\Gamma(\tfrac{D-1}{2})\sqrt{\pi}} \cdot 
\int_{-1}^1  (1-x^2)^{s+(D-3)/2} \dif x = \frac{\Gamma(\tfrac{D}{2})}{\Gamma(\tfrac{D-1}{2})\sqrt{\pi}}
\cdot \sqrt{\pi} \frac{\Gamma(s+\frac{D}{2}-\frac{1}{2})}{\Gamma(s+\frac{D}{2})}.
\end{array}
\end{equation*}
\end{linenomath*}
To simplify the above we apply the formula (1) from~\citep{tricomi1951asymptotic}:
\begin{linenomath*}
\begin{equation*}
\frac{\Gamma(z+\alpha)}{\Gamma(z+\beta)}=z^{\alpha-\beta}(1+\frac{(\alpha-\beta)(\alpha+\beta-1)}{2z}+O(|z|^{-2})),
\end{equation*}
\end{linenomath*}
with $\alpha,\beta$ fixed so that $\alpha+\beta=1$ (so only the error term of order $O(|z|^{-2})$ remains), and get the following
\begin{equation}
\begin{array}{r@{}l}
{\Gamma(\frac{D}{2})}/{\Gamma(\frac{D-1}{2})}&=
\frac{\Gamma((\frac{D}{2}-\frac{3}{4})+\frac{3}{4})}{\Gamma((\frac{D}{2}-\frac{3}{4})+\frac{1}{4})}
\approx \left(\frac{D}{2}-\frac{3}{4}\right)^{\frac{1}{2}}
\\
{\Gamma(s+\frac{D}{2}-\frac{1}{2})}/{\Gamma(s+\frac{D}{2})}
&\approx \left(s+\frac{D}{2}-\frac{3}{4}\right)^{-\frac{1}{2}}.
\end{array}
\end{equation}
Summarizing,
\begin{linenomath*}
\begin{equation*}
\phi_D(s) \approx
\frac{(\tfrac{D}{2}-\tfrac{3}{4})^{1/2}}{(s+\tfrac{D}{2}-\tfrac{3}{4})^{1/2}}
=
(1+\tfrac{4s}{2D-3})^{-1/2}.
\end{equation*}
\end{linenomath*}
\end{proof}

In general one can obtain the iterative direct formulas for function $\phi_D$ 
with the use of $\erf$ and modified Bessel functions of the first kind $I_0$ and $I_1$, but for large $D$ they are of little numerical value. We consider here only the special case $D=2$ since it is used in the paper for illustrative reasons in the latent for the MNIST data set. Since
we have the equality~\citep[(8.406.3) and (9.215.3)]{gradshteyn2014table}:
\begin{linenomath*}
\begin{equation*}
\phi_2(s)=\1F1(\tfrac{1}{2},1,-s)= e^{-\frac{s}{2}} I_0\left(\frac{s}{2}\right),
\end{equation*}
\end{linenomath*}
to practically implement $\phi_2$ we apply the approximation of $I_0$ from~\citet[p.~378]{abramowitz1964handbook} given in the following remark.


\begin{remark} \label{pr:p}
Let $s \geq 0$ be arbitrary and let $t=s/7.5$. Then
\begin{linenomath*}
\begin{equation*}
\begin{array}{l}
\phi_2(s)\approx 
\begin{cases}
\begin{array}{l}
\!\!e^{-\tfrac{s}{2}}\! \cdot \! (1\+3.51562t^2\+3.08994t^4 \+1.20675t^6+ 0.26597t^8\+0.03608t^{10}\+0.00458t^{12}) \\
\quad\text{ for }s \in [0,7.5], \\
\!\!\sqrt{\frac{2}{s}} \! \cdot \! (0.398942\+0.013286t^{-1} 
\+ 0.002253t^{-2} \!-\! 0.001576t^{-3} 
\+ 0.00916t^{-4}\!\!-\!0.020577t^{-5} \\
\quad+ 0.026355t^{-6}\!\!-\!0.016476t^{-7} 
\+ 0.003924t^{-8}) 
\quad\text{ for }s \geq 7.5.
\end{array}
\end{cases}
\end{array}
\end{equation*}
\end{linenomath*}
\end{remark}

\section{Cramer-Wold kernel}\label{sec:cw-kernel} \label{ap:CW}

In this section we first formally define the Cramer-Wold metric for arbitrary measures, and later show that it is given by a characteristic kernel which has closed-form for spherical Gaussians. For more information on kernels, and kernel embedding of distributions we refer the reader to~\citet{muandet2017kernel}.

Let us first introduce the general definition of the $\CW$-metric. To do so we generalize the notion of smoothing for arbitrary measures $\mu$ by the formula:
\begin{linenomath*}
\begin{equation*}
\reg_\gamma(\mu)=\mu * N(0,\gamma I),
\end{equation*}
\end{linenomath*}
where $*$ denotes the convolution operator for two measures, and we identify the normal density $N(0,\gamma I)$ with the measure it introduces. It is well-known that the resulting measure has the density given by 
\begin{linenomath*}
\begin{equation*}
x \to \int N(x,\gamma I)(y) d\mu(y).
\end{equation*}
\end{linenomath*}
Clearly 
\begin{linenomath*}
\begin{equation*}
\reg_{\gamma}(N(0,\alpha I))=N(0,(\alpha+\gamma)I)).
\end{equation*}
\end{linenomath*}
Moreover, by applying the characteristic function one obtains that if the smoothing of two measures coincide, then the measures also coincide:
\begin{equation}\label{eq:CW1}
\reg_\gamma (\mu)=\reg_\gamma(\mu)
\implies \mu=\nu.
\end{equation}

We also need to define the transport of 
the density by the projection $x \to v^Tx$, where $v$ is chosen from the unit sphere $S_D$. The definition is formulated so that if $\X$ is a random vector with density $f$, then $f_v$ is the density of the random vector $\X_v:=v^T \X$. Then we have
\begin{linenomath*}
\begin{equation*}
f_v(r)=\int_{y:y-rv \perp v} f(z) d_{D-1}(z),
\end{equation*}
where $d_{D-1}$ denotes the $D-1$-dimensional Lebesgue measure.
In general, if $\mu$ is a measure on $\R^D$, then $\mu_v$ is the measure defined on $\R$ by the formula
\begin{equation*}
\mu_v(A)=\mu(\{x: v^T x \in A\}).
\end{equation*}
\end{linenomath*}

Since, if a random vector $\X$  has the density $N(a,\gamma I)$, then the random variable $\X_v$ has the density $N(v^Ta,\alpha)$, we directly conclude that 
\begin{linenomath*}
\begin{equation*}
N(a,\gamma I)_v=N(v^Ta,\gamma).
\end{equation*}
\end{linenomath*}

It is also worth noticing, that due to the fact that the projection of a Gaussian is a Gaussian, the smoothing and projection operators commute, i.e.:
\begin{linenomath*}
\begin{equation*}
\reg_\gamma (\mu_v)=(\reg_\gamma \mu)_v.
\end{equation*}
\end{linenomath*}

Given fixed $\gamma>0$, the two above notions allow us to formally define the $\CW$-distance of two measures $\mu$ and $\nu$ by the formula
\begin{equation} \label{eq:dcw}
d_{\CW}^2(\mu,\nu)=\int_{S_D} \|\reg_\gamma(\mu_v)-\reg_\gamma(\nu_v)\|^2_{L_2} d\sigma_D(v).
\end{equation}
First observe that this implies that $\CW$-distance is given by the kernel function
\begin{linenomath*}
\begin{equation*}
k(\mu,\nu)=\int_{S_D} \il{\reg_\gamma(\mu_v),\reg_\gamma(\nu_v)}_{L_2} d\sigma_D(v).
\end{equation*}
\end{linenomath*}
Let us now prove that the function $d_{\CW}$ defined by equation~\eref{eq:dcw} is a metric (which, in the kernel function literature means that the kernel is characteristic). 

\begin{theorem}
The function $d_{\CW}$ is a metric.
\end{theorem}

\begin{proof}
Since $d_{\CW}$ comes from a scalar product, we only need to show that if a distance of two measures is zero, the measures coincide.

So let $\mu,\nu$ be given measures such that $d_{\CW}(\mu,\nu)=0$. This implies that 
\begin{linenomath*}
\begin{equation*}
\reg_\gamma(\mu_v)=\reg_\gamma(\nu_v).
\end{equation*}
\end{linenomath*}
By \eref{eq:CW1} this implies that $\mu_v=\nu_v$. Since this holds for all $v \in S_D$, by the Cramer-Wold Theorem we obtain that $\mu=\nu$.
\end{proof}

Thus we can summarize the above by saying that the Cramer-Wold kernel is a characteristic kernel which, by the definition and \eqref{eq:fr}, has the closed-form the scalar product of two radial Gaussians given by:
\begin{equation} \label{eq:formula}
\il{N(x,\alpha I),N(y,\beta I)}_{\CW}=\tfrac{1}{\sqrt{2\pi (\alpha+\beta+2\gamma)}}\phi_D\left(\tfrac{\|x-y\|^2}{2(\alpha+\beta+2\gamma)}\right).
\end{equation}

\begin{remark}
Observe, that except for the Gaussian kernel it is the only kernel which has the closed-form for the spherical Gaussians which, as we discuss in the next section, is important, as the RBF (Gaussian) kernels cannot be successfully applied in AutoEncoder based generative models. The reason is that the  derivative of Gaussian decreases too fast, and therefore it does not enforce the proper learning of the model (see also~\citet[Section 4, WAE-GAN and WAE-MMD specifics]{tolstikhin2017wasserstein}).
\end{remark}

\section{Cramer-Wold AutoEncoder model (CWAE)}\label{se:CWAE}

This section is devoted to the construction of CWAE. Since we base our construction on the AutoEncoder, to establish notation let us formalize it here.

\paragraph{AutoEncoder.} Let $X=(x_i)_{i=1..n} \subset \R^N$ be a given data set. The basic aim of AE is to transport the data to a typically, but not necessarily, lower dimensional latent space $\Z=\R^D$ while minimizing the reconstruction error. Thus, we search for an encoder $\E:\R^n \to \Z$ and decoder $\D:\Z \to \R^n$ functions, which minimizes the mean squared error $MSE(X;\E,\D)$ on $X$ and its reconstructions $\D(\E x_i)$. 

\paragraph{AutoEncoder based generative model.} CWAE, similarly to WAE, is a classical AutoEncoder model with modified cost function which forces the model to be generative, i.e. ensures that the data transported to the latent space comes from the (typically Gaussian) prior $f$. This statement is formalized by the following important remark, see also~\citet{tolstikhin2017wasserstein}.

\begin{remark} \label{re:wazny}
Let  $\X$ be an $N$-dimensional random vector from which our data set was drawn, and let $\Y$ be a~random vector with 
a density $f$ on latent $\Z$. 

Suppose that we have constructed
functions $\E:\R^N \to \Z \text{ and }\D:\Z \to \R^N$ (representing the encoder and the decoder) such that\footnote{We recall that for function (or in particular random vector) $\X:\Omega \to \R^D$ by $\mathrm{image}(\X)$ we denote the set consisting of all possible values $\X$ can attain, i.e. $\{\X(\omega):\omega \in \Omega\}$.} 
\begin{enumerate}
\item $\D(\E x)=x$ for $x \in \mathrm{image}(\X)$,
\item random vector $\E\X$ has the distribution $f$.
\end{enumerate}
Then by the point 1 we obtain that $\D(\E \X)=\X$, and therefore 
\begin{linenomath*}
\begin{equation*}
\D \Y \text{ has the same distribution as } \D(\E \X) =\X.
\end{equation*}
\end{linenomath*}
This means that to produce samples from $\X$ we can instead produce samples from $\Y$ and map them by the decoder $\D$.
\end{remark}
Since an estimator of the image of the random vector $\X$ is given by its sample $X$, we conclude that a generative model is correct if it has small reconstruction error and resembles the prior distribution in the latent. Thus, to construct a generative AutoEncoder model (with Gaussian prior), we add to its cost function a measure of distance of a given sample from normal distribution.

\paragraph{CWAE cost function.} Once the crucial ingredient of CWAE is ready, we can describe its cost function. To ensure that the data transported to the latent space $\Z$ are distributed according to the standard normal density, we add to the cost function logarithm\footnote{We take the logarithm of the Cramer-Wold distance to improve balance between the two terms in the objective function.} of the Cramer-Wold distance from standard multivariate normal density $d^2_{\CW}(X,N(0,I))$:
\begin{linenomath*}
\begin{equation*} \label{eq:cost}
\mathrm{cost}(X;\E,D)= MSE(X;\E,\D) + \lambda \log d^2_\CW(\E X,N(0,I)).
\end{equation*}
\end{linenomath*}

\begin{figure*}[htb]
\begin{tabular}{@{}c@{}c@{}c@{}c@{}}
 &  Test interpolation &  Test reconstruction &  Random sample \\
 \rotatebox{90}{ \qquad\qquad VAE} & \,
\includegraphics[height=3.9cm]{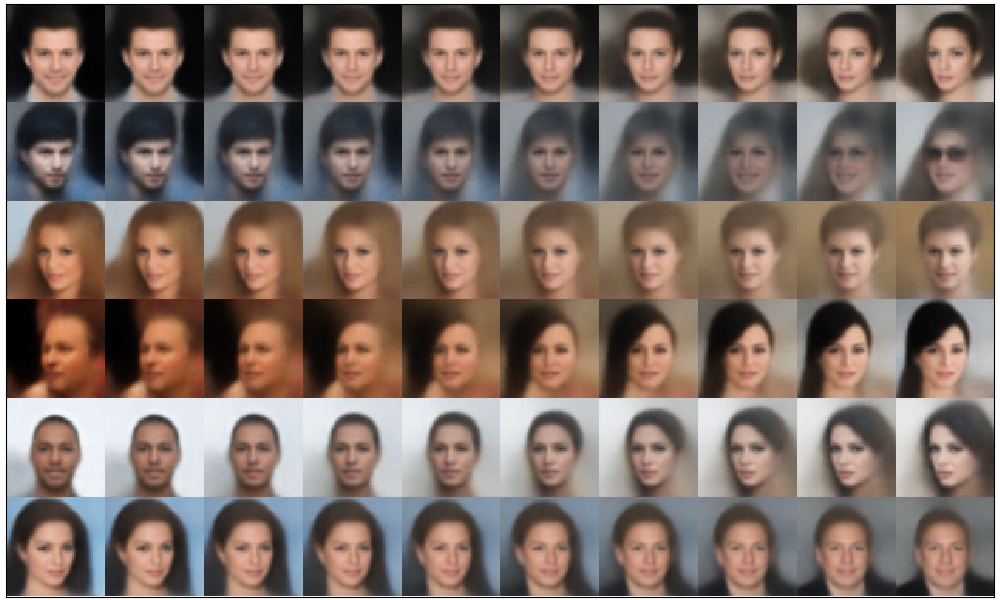} & \,
\includegraphics[height=3.9cm]{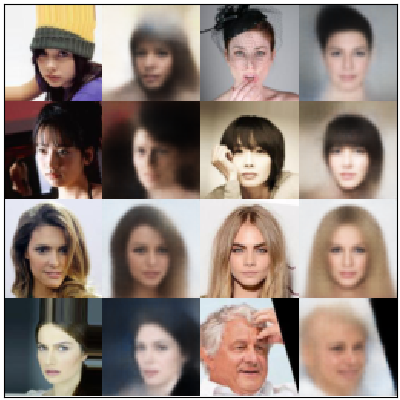}& \, 
\includegraphics[height=3.9cm]{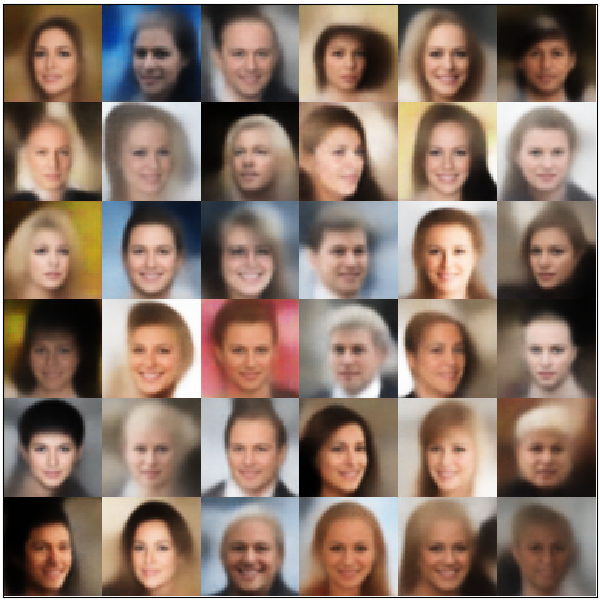} \\ 
 \rotatebox{90}{ \qquad\qquad WAE-MMD} & \,
\includegraphics[height=3.9cm]{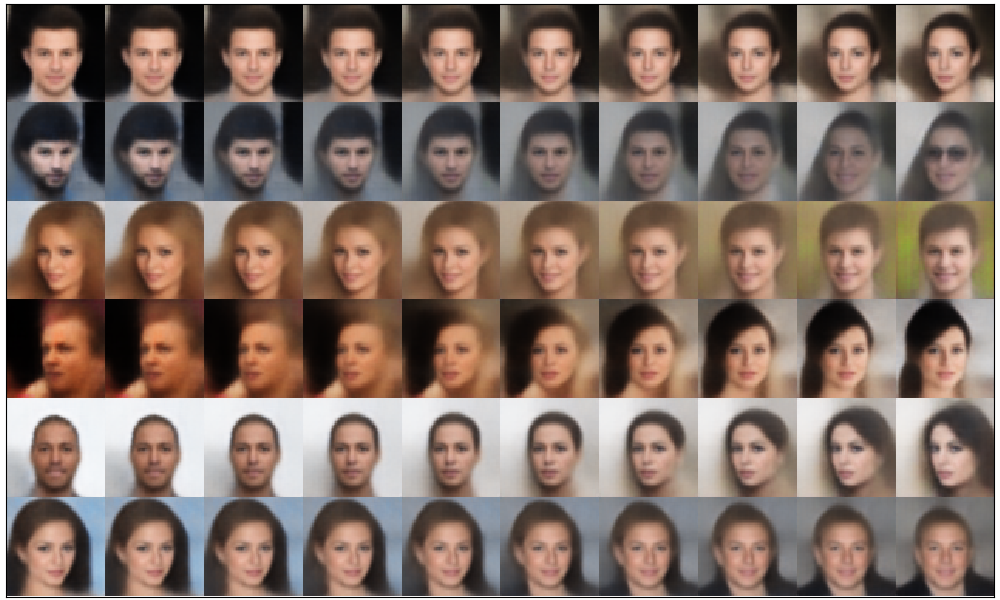} & \,
\includegraphics[height=3.9cm]{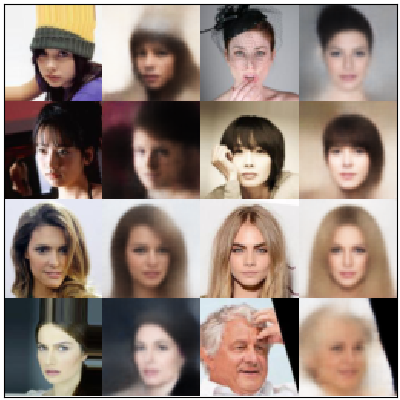}& \, 
\includegraphics[height=3.9cm]{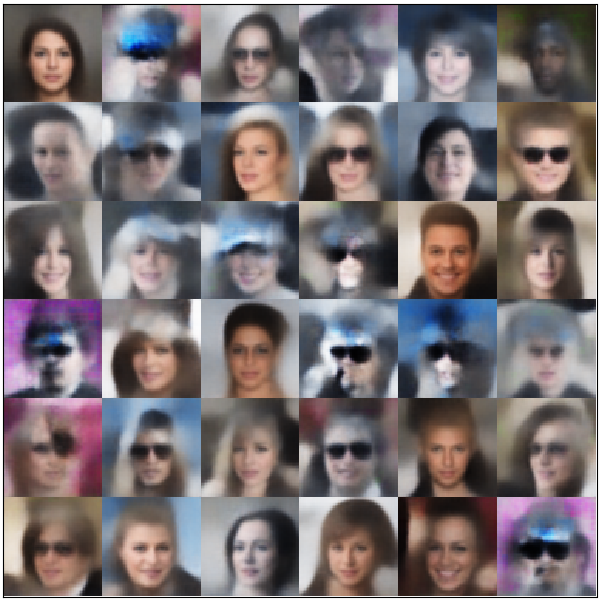} \\ 
 \rotatebox{90}{ \qquad\qquad SWAE} & \,
\includegraphics[height=3.9cm]{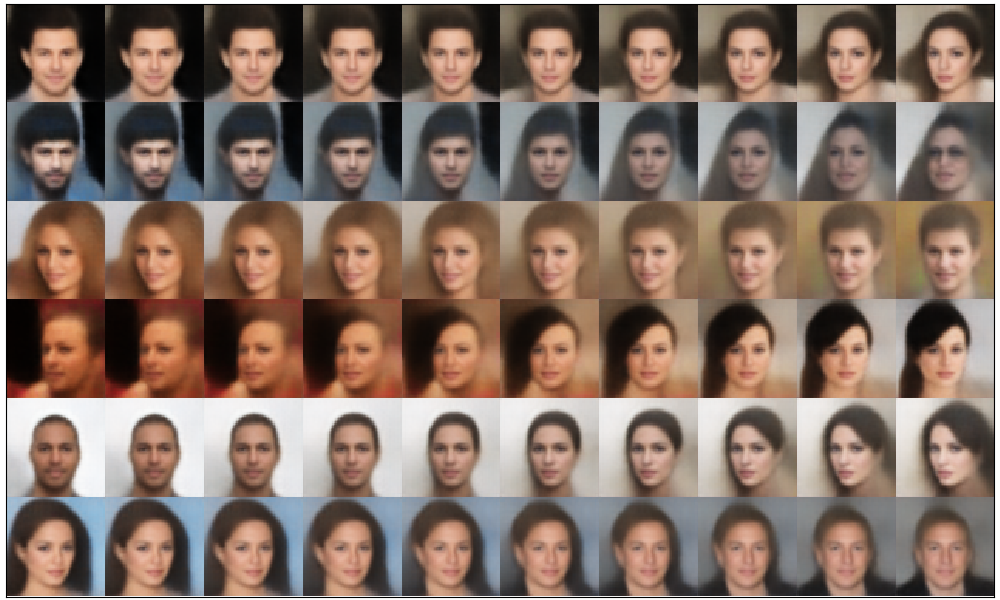} & \,
\includegraphics[height=3.9cm]{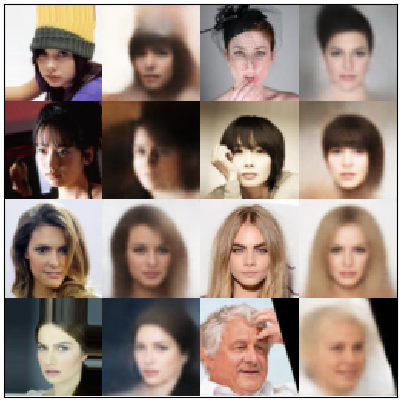}& \, 
\includegraphics[height=3.9cm]{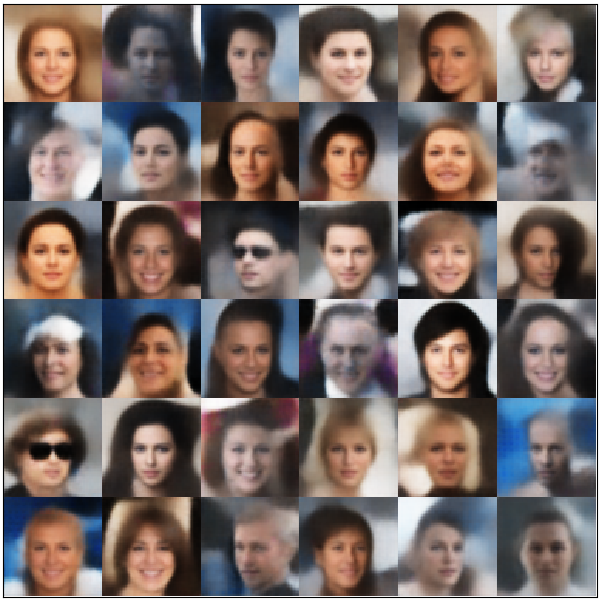} \\ 
\rotatebox{90}{ \qquad\qquad CWAE} & \,
\includegraphics[height=3.9cm]{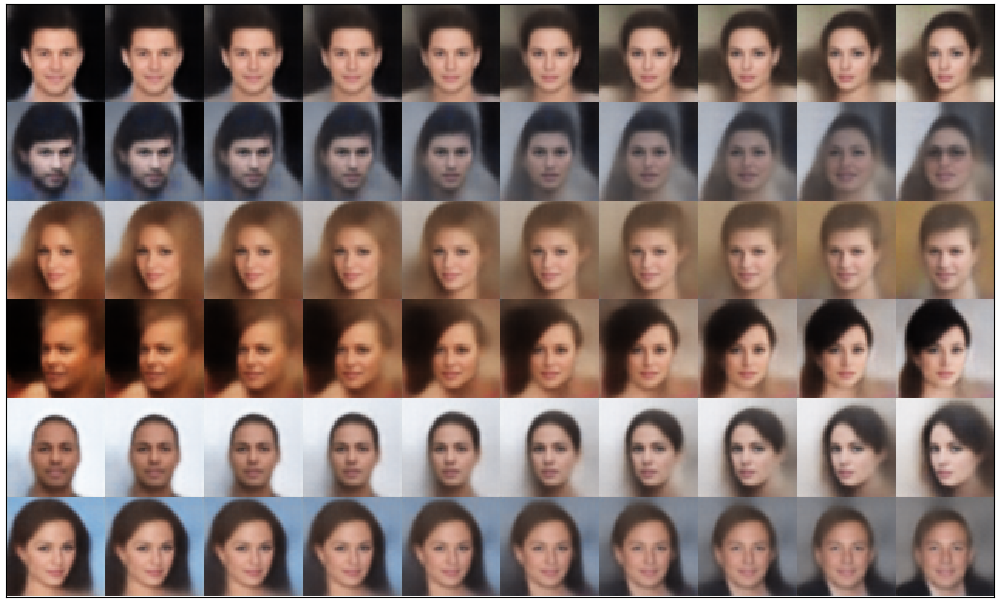} & \,
\includegraphics[height=3.9cm]{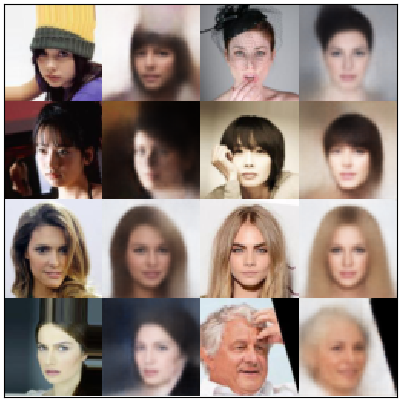}& \, 
\includegraphics[height=3.9cm]{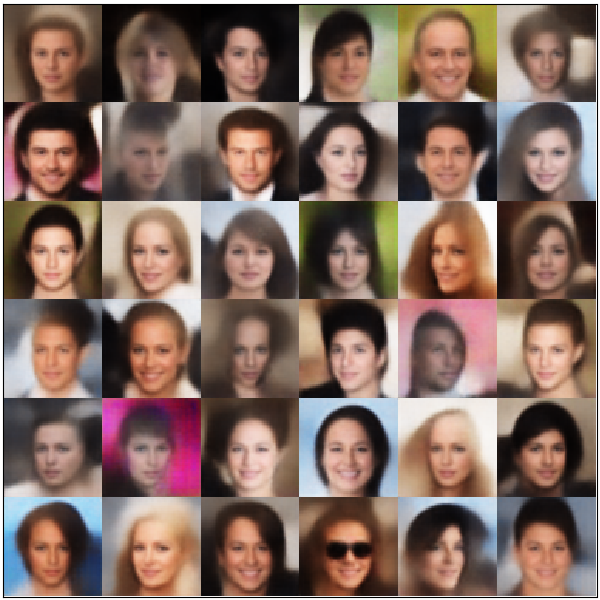} \\ 
\end{tabular}
\caption{Results of VAE, WAE-MMD, SWAE, and CWAE models trained on CELEB~A dataset using the WAE architecture from~\citep{tolstikhin2017wasserstein}. In “test reconstructions” odd rows correspond to the real test points.}
\label{fig:celeb}
\end{figure*}

Since the use of special functions involved in the formula for Cramer-Wold distance might be cumbersome, we apply in all experiments (except for the illustrative 2D case) the asymptotic form \eref{eq:Phi} of function $\phi_D$:
\begin{linenomath*}
\begin{equation*} 
\begin{array}{r@{}l}
2\sqrt{\pi}d^2_\CW(&X) \approx
\tfrac{1}{n^2}\sum_{ij} 
(\gamma_n+\tfrac{\|x_i-x_j\|^2}{2D-3})^{-1/2}+(1+\gamma_n)^{-1/2}- 
\tfrac{2}{n} \sum_{i} (\gamma_n+\tfrac{1}{2}+\tfrac{\|x_i\|^2}{2D-3})^{-1/2},
\end{array}
\end{equation*}
\end{linenomath*}
where $\gamma_n=(\frac{4}{3n})^{2/5}$ is chosen by the Silverman's rule of thumb~\citep{silverman1986density}.

\paragraph{Comparison with WAE and SWAE models.} 
Finally, let us briefly recapitulate differences between the introduced CWAE, WAE variants~\citep{tolstikhin2017wasserstein} and SWAE~\citep{kolouri2018sliced}. In contrast to WAE-MMD and SWAE, CWAE model \emph{does not} require sampling from  normal distribution (as in WAE-MMD) or over slices (as in SWAE) to evaluate its cost function, and in this sense uses a closed formula cost function. In contrast to WAE-GAN, our objective does not require a separately trained neural network to approximate the optimal transport function, thus avoiding pitfalls of adversarial training. 

\paragraph{CWAE vs WAE-MMD.} 
We shall now compare CWAE model to WAE-MMD. In particular we show that CWAE can be seen as the intersection of the sliced-approach together with MMD-based models. WAE-MMD model uses approximations, while CWAE uses a closed form, which bears an impact on training. It results in more leveled drop of distance weight, with even negative values in case of of WAE-MMD estimator, see Fig.~\ref{fig:synthetic-latent}. See paragraph below on using a logarithm in cost function.

Since both WAE and CWAE use kernels to discriminate between sample and normal density, to compare the models we first describe the WAE model. WAE cost function for a given characteristic kernel $k$ and sample $X=(x_i)_{i=1..n} \subset \R^D$ (in the $D$-dimensional latent) is given by 
\begin{equation}\label{eq:wae-cost}
\text{WAE cost}=MSE+\lambda \cdot d^2_k(X,Y),
\end{equation}
where $Y=(y_i)_{i=1..n}$ is a sample from the standard normal density $N(0,I)$, and $d^2_k(X,Y)$ denotes the kernel-based distance between the probability distributions representing $X$ and $Y$, that is $\frac{1}{n}\sum_i \delta_{x_i}$ and $\frac{1}{n}\sum_i \delta_{y_i}$, where $\delta_z$ denotes the atom Dirac measure at $z \in \R^D$.
The inverse multiquadratic kernel IMQ $k$ is chosen as default
\begin{linenomath*}
\begin{equation*}
k(x,y)=\frac{C}{C+\|x-y\|_2^2},
\end{equation*}
\end{linenomath*}
where in experiments in~\citet{tolstikhin2017wasserstein}, a value $C=2D \sigma^2$ was used, where $\sigma$ is the hyper-parameter denoting the size of the normal density. Thus the model has hyper-parameters $\lambda$ and $\sigma$, which were chosen to be $\lambda=10,\sigma^2=1$ in MNIST, $\lambda=100,\sigma^2=2$ in CELEB~A. 

Now let  us describe the CWAE model. CWAE cost function for a sample $X=(x_i)_{i=1..n} \subset \R^D$ (in the $D$-dimensional latent) is given by 
\begin{equation}\label{eq:cwae-cost}
\text{CWAE cost}=MSE + \lambda  \log d^2_{\CW}(X,N(0,I)),
\end{equation}
where distance between the sample and standard normal distribution is taken with respect to the Cramer-Wold kernel with a regularizing hyperparameter $\gamma$ given by the Silverman's rule of thumb (the motivation for such a choice of hyper-parameters is explained in Section~\ref{sec:cw_distance}).

\paragraph{CWAE cost and the use of logarithm}
In Eq.~\eqref{eq:cwae-cost} we have introduced the CWAE cost which differs from the WAE model cost Eq.~\eqref{eq:wae-cost}, by the use of a logarithm function. The model works perfectly when no function modifier is used, but we have noticed that it gained in speed when logarithm was applied, see Figure~\ref{fig:cw_vs_logcw}. One might ask why? It seems that two factors are the most important.

\begin{figure}[htb]
\centering
\begin{tabular}{@{}c@{}c@{}}
\includegraphics[width=0.76\textwidth]{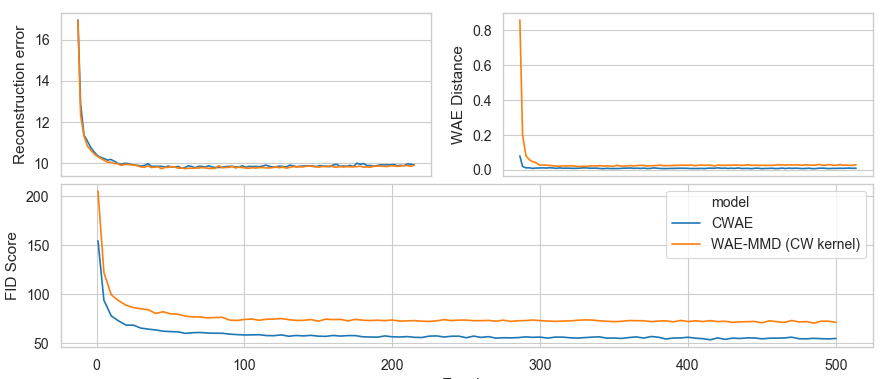}
\end{tabular}
\caption{Competition between CWAE and WAE-MMD witch CW kernel on Fashion-MNIST dataset.}
\label{fig:cw_vs_logcw}
\end{figure}

First, it might be noticed, that during the first few iterations of model training, it is normal for the variation of the errors to be high, see Figure~\ref{fig:cw_vs_logcw}. In case of CWAE, the $D_{cw}$ cost has around $10$ times larger scale than $d_{k}$ cost has in case of WAE. The $\log$ tones it done substantially, increasing the stability of learning, which is not needed in WAE. The network finds a smoother way to increase the normality of the latent space, thus speeding up training. 

Second, it might be hypothesized, that at the beginning of training values in the latent are more uniform, becoming more and more normal (let's suppose a normal prior). A synthetic data experiment showing this is given in Figure~\ref{fig:synthetic-latent}. The logarithmic cost drops-off much quicker pulling the model towards quicker minimization. At the same time, when the WAE-MMD with modified cost $\dots+d_k^2(\cdot,\cdot)$, see Eq.~\eqref{eq:wae-cost}, to $\dots+\log{}d_k^2(\cdot,\cdot)$ results in much steeper and more irregular descent. The WAE-MMD cost is closer to zero, and may sometimes be even negative as noted in~\citet["\dots penalty used in WAE-MMD is not precisely the population MMD, but a sample based U-statistic\dots if the population MMD is zero, it necessarily needs to take negative values from time to time."]{tolstikhin2018waegithub}. Thus the $\log$ version is not suitable for the WAE-MMD version, which coincides with experiments.

\begin{figure}[htb]
\centering
\begin{tabular}{@{}c@{}c@{}}
\includegraphics[width=0.44\textwidth]{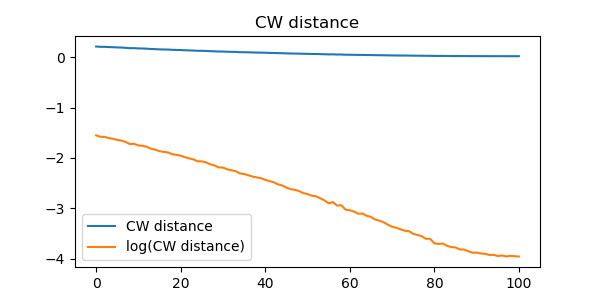}&
\includegraphics[width=0.44\textwidth]{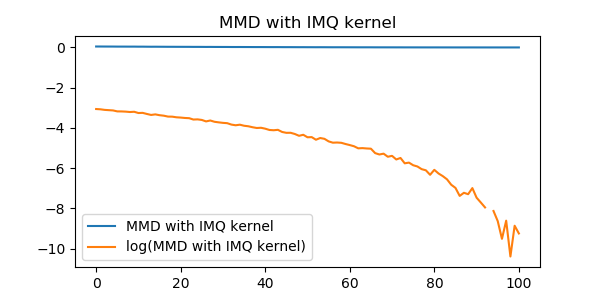}
\end{tabular}
\caption{Synthetic data in the latent and the distance from prior cost: left for the CWAE model, right for WAE-MMD. On the horizonal axis is the share of $z\sim{N}(0,1)$ in otherwise uniform data. The blue curves represent a standard model (without logarithm), while the orange one with a logarithm used.}
\label{fig:synthetic-latent}
\end{figure}

\paragraph{Thus we stress the following important differences}
\begin{itemize}
\item Due to the properties of Cramer-Wold kernel, we are able to substitute the sample estimation of $d^2_k(X,N(0,I))$ given in WAE-MMD by $d^2_{\CW}(X,Y)$ by its exact formula.
\item CWAE, as compared to WAE, it is less sensitive to the choice of parameters: 
\begin{enumerate} 
\item The choice of regularization hyper-parameter is given by the Silverman's rule of thumb, and depends on the sample size (contrary to WAE-MMD, where the hyper-parameters are chosen by hand, and in general does not depend on the sample size).
\item  In our preliminary experiments we have observed that frequently (like in the case of log-likelihood), taking the logarithm of the non-negative factors of the cost function, which we aim to minimize to zero, improves the learning. Motivated by this and an analysis (see above), the CWAE cost uses logarithm of Cramer-Wold distance to balance the MSE and divergence terms. It turned out that in most cases it is enough to set in Eq.~\eqref{eq:cwae-cost} the parameter $\lambda=1$. 
\end{enumerate}
\end{itemize}

Summarizing, CWAE model, contrary to WAE-MMD, is less sensitive to the choice of parameters. Moreover, since we do not have the noise in the learning process given by the random choice of the sample $Y$ from $N(0,I)$, the learning should be more stable. As a consequence,  CWAE generally learns faster then WAE-MMD, and has smaller standard deviation of the cost-function during the learning process. Detailed results of the experiments in the case of CELEB~A  are presented in Figure~\ref{fig:conv1}. 
Moreover, for better comparison, we verified how the learning process looks like in the case of original WAE-MMD architecture form \citep{tolstikhin2017wasserstein}, see Figure~\ref{fig:conv1}. 

\paragraph{Generalized Cramer-Wold kernel.} 
In this paragraph we show that  asymptotically (with respect to dimension $D$), Cauchy kernel used in WAE-MMD can in fact be seen as the sliced kernel (where as slices we use two-dimensional subspaces). To do so we need the probability measure on $d$-dimensional linear subspaces of $\R^D$, see~\citet{mattila1999geometry}. One can do it either directly with the definition of Grassmanian, or alternatively describe it with the orthonormal basis, for integration over orthonormal matrices~\citep{aubert2003invariant,braun2006invariant}. 

Now we define the $d$-dimensional sliced Cramer-Wold kernel by the formula
\begin{linenomath*}
\begin{equation*}
k_d(\mu,\nu)=\int_{G(d,D)} \il{\reg_\gamma(\mu_v),\reg_\gamma(\nu_v)}_{L_2} d\gamma_{d,D}(v),
\end{equation*}
\end{linenomath*}
where $\gamma_{d,D}$ denotes the respective Radon probability measure on $G(d,D)$. Equivalently we can integrate over orthonormal sequences in $\R^D$ of length $d$:
\begin{linenomath*}
\begin{equation*}
O_d(\R^D)=\{(v_1,\ldots,v_d) \in (\R^D)^d: \|v_i\|=1, v_i \perp v_j\}.
\end{equation*}
\end{linenomath*}
The normalized (invariant with respect to orthonormal transformations) measure on $O_d$ we denote with $\theta_d$. Observe that for $d=1$ we obtain normalized integration over the sphere. 

Then we obtain that $k_d$ can be equivalently defined as
\begin{linenomath*}
\begin{equation*}
k_d(\mu,\nu)=\int_{O_d} \il{\reg_\gamma(\mu_v),\reg_\gamma(\nu_v)}_{L_2} d\theta_{d}(v).
\end{equation*}
\end{linenomath*}
Let us first observe that for gaussian densities the formula for $k_d$ can be slightly simplified, namely
\begin{linenomath*}
\begin{equation*}
\begin{array}{r@{}l}
k_d(N(x,\alpha I),N(y,\beta I))&=\int_{O_d} N(v^T(x-y),(\alpha+\beta+2\gamma)I_d)(0) d\theta_d(v)\\
&=\int_{O_d} \prod \limits_{i=1}^d N(v_i^T(x-y),\alpha+\beta+2\gamma)(0) d\theta_d(v).
\end{array}
\end{equation*}
\end{linenomath*}
Now if we define
\begin{linenomath*}
\begin{equation*}
\Phi_D^d(s,h)
=\int_{O_d}N(v^T se_1,hI_d)(0) d\theta_d(v),
\end{equation*}
\end{linenomath*}
where $e_1 \in \R^D$ is the first unit base vector, we obtain that 
the kernel-product reduces to computation of the scalar function $\Phi_D$
\begin{linenomath*}
\begin{equation*}
k_d(N(x,\alpha I),N(y,\beta I))=
\Phi_D^d(\|x-y\|,\alpha+\beta+2\gamma).
\end{equation*}
\end{linenomath*}
The crucial observation needed to proceed further is that the measure space $(O_d(\R^D),\theta_d)$ can be approximated by $(\R^D,N(0,I/D))^d$. This follows from the fact, that if $v_1,\ldots,v_d$ are drawn from the density $N(0,I/D)$, then for sufficiently large $D$ we have $\|v_i\| \approx 1$ and $\il{v_i ,v_j} \approx 0$ for $i \neq j$.

\begin{theorem}
We have
\begin{linenomath*}
\begin{equation*}
\Phi^d_D(s,h) \to 
(2\pi)^{-d/2} \cdot (h+s^2/D)^{-d/2}.
\end{equation*}
\end{linenomath*}
\end{theorem}

\begin{proof}
By the observation stated before the theorem, we have
\begin{linenomath*}
\begin{equation*}
\begin{array}{r@{\,}l}
\Phi_D^d(s,h)
&=\int_{O_d} \prod \limits_{i=1}^d N(v_i^Tse_1,h)(0) d\theta_d(v)\approx \int \limits_{(\R^D)^d} \prod \limits_{i=1}^d N(v_i^Tse_1,h)(0) N(0,\frac{I}{D})(v_i) dv_1 \ldots dv_d\\
&= \prod \limits_{i=1}^d
\int \limits_{\R^D} N(v_i^Tse_1,h)(0) N(0,\frac{I}{D})(v_i) dv_i.
\end{array}
\end{equation*}
\end{linenomath*}
Thus it suffices to compute each component of the above formula. To do so, we denote by $N_k$ the $k$-dimensional normal density, and get

\begin{linenomath*}
\begin{equation*}
\begin{array}{r@{\,}l}
\int_{\R^D} N_1(s\il{v,e_1},h)(0) \cdot N_D&(0,\tfrac{1}{D}I)(v)dv=\\
&=\int_{-\infty}^{\infty} N_1(0,h)(st) \frac{N_D(0,\frac{1}{D}I)(te_1)}{N_{D-1}(0,\frac{1}{D}I)(0)} \int_{\R^{D-1}} N_{D-1}(0,\tfrac{1}{D}I)(w) dw \, dt\ \\
&=\int_{-\infty}^{\infty} N_1(0,h)(st)\frac{N_D(0,\frac{1}{D}I)(te_1)}{N_{D-1}(0,\frac{1}{D}I)} dt\\
&=\int_{-\infty}^{\infty} \frac{1}{\sqrt{2\pi h}}\exp(-\frac{1}{2h}(st)^2)
\cdot\frac{\sqrt{D}}{\sqrt{2\pi}}\exp(-\frac{1}{2}Dt^2)dt
\sqrt{2\pi}\sqrt{\frac{h}{s^2+hD}}\\
&=\frac{1}{\sqrt{2\pi}} \frac{1}{\sqrt{h+s^2/D}},
\end{array}
\end{equation*}
\end{linenomath*}

which yields the assertion of the theorem.
\end{proof}

As a direct consequence, we obtain the following asymptotic formula (with dimension $D$ large) of the generalized Cramer-Wold kernel of two spherical Gaussians
\begin{linenomath*}
\begin{equation*}
\begin{array}{r@{}l}
k_d(N(x,&\alpha I),N(y,\beta I)) \approx (2\pi)^{-d/2} \cdot (\alpha+\beta+2\gamma+\|x-y\|^2/D)^{-d/2}.
\end{array}
\end{equation*}
\end{linenomath*}
Observe, that with $d=2$ we obtain the standard inverse multiquadratic kernel.



\section{Experiments} \label{se:ex}

In this section we empirically validate the proposed CWAE\footnote{The code is available \url{https://github.com/gmum/cwae}.} model on standard benchmarks for generative models: CELEB~A, CIFAR-10 and MNIST, FashionMNIST. We compare proposed CWAE model with WAE-MMD~\citep{tolstikhin2017wasserstein} and SWAE~\citep{kolouri2018sliced}. As we shall see, our results match or even exceed those of WAE-MMD and SWAE, while using a closed-form cost function (see the previous section for a more detailed discussion). The rest of this section is structured as follows. In Section~\ref{se:ex:qual} we report results on standard qualitative tests, as well as visual investigations of the latent space. In Section~\ref{se:ex:quant} we will turn our attention to quantitative tests using Fr\'{e}chet Inception Distance and other metrics~\citep{heusel2017gans}.


\subsection{Experimentation setup}

In the experiment we use two basic architecture types. Experiments on MNIST and Fashion-MNIST use a feed-forward network for both encoder and decoder, and a 8 neuron latent layer, all using ReLU activations. For CIFAR-10, and CELEB~A data sets we use convolution-deconvolution architectures. Please refer to Section ~\ref{app:architectures} for full details.

\begin{table*}[htb]
\normalsize
\caption{Comparison of different architectures on the MNIST, Fashion-MNIST, CIFAR-10 and CELEB~A datasets. 
All models outputs except AE are similarly close to the normal distribution. CWAE achieves the best value of FID score (lower is better). 
}
\begin{center}
{\small
\begin{tabular}[width=\textwidth]{lllrrrrr}
\toprule
Data set & Method & Learning & $\lambda$ & Skewness & Kurtosis & Rec. & FID \\
 &  & rate & & & (normalized) & error & score\\
\midrule             
MNIST   & AE & 0.001 & - &  1197.24 & 878.07 & 11.19 & 52.74 \\
        & VAE & 0.001 & - &  0.43 & 0.77 & 18.79 & 40.47 \\
        & SWAE & 0.001 & 1.0 &  6.01 & 10.72 & 10.99 & 29.76 \\
        & WAE-MMD & 0.0005 & 1.0 &  11.70 & 8.34 & 11.14 & 27.65 \\
        & CWAE & 0.001 & 1.0 &  12.21 & 35.88 & 11.25 &  \bf 23.63 \\
\midrule             
FASHION & AE & 0.001 & - &  140.21 & 85.58 & 9.87 & 81.98 \\
MNIST   & VAE & 0.001 & - &  0.20 & 4.86 & 15.41 & 64.98 \\
        & SWAE & 0.001 & 1.0 &  8.22 & 11.58 & 9.91 & 74.32 \\
        & WAE-MMD  & 0.001 & 1.0 &  22.17 & 37.96 & 9.90 & 69.16 \\
        & CWAE & 0.001 & 1.0 &  30.38 & 98.58 & 9.84 & \bf  57.06 \\
\midrule             
CIFAR10 & AE & 0.001 & - &  $2.5\times10^5$ & $1.7\times10^4$ & 24.67 & 269.09 \\
        & VAE & 0.001 & - &  35.81 & 3.67 & 63.77 & 172.39 \\
        & SWAE & 0.001 & 1.0 &  517.32 & 121.17 & 25.42 & 141.91 \\
        & WAE-MMD  & 0.001 & 1.0 &  1105.73 & 2097.14 & 25.04 & 129.37 \\
        & CWAE & 0.001 & 1.0 &  176.60 & 1796.66 & 25.93 & \bf  120.02 \\
\midrule             
CELEB~A & AE & 0.001 & - &  $4.6\times10^9$ & $2.6\times10^8$ & 86.41 & 353.50 \\
        & VAE & 0.001 & - &  43.72 & 171.66 & 110.87 & 60.85 \\
        & SWAE & 0.0001 & 100.0 &  141.17 & 222.02 & 85.97 & 53.85 \\
        & WAE-MMD & 0.0005 & 100.0 &  162.67 & 604.09 & 86.38 & 51.51 \\
        & CWAE & 0.0005 & 5.0 &  130.08 & 542.42 & 86.89 & \bf 49.69 \\
\bottomrule
\end{tabular}
}
\end{center}
\label{tab:comp}
\end{table*}

\begin{figure*}[htb]
\centering
\includegraphics[width=0.8\textwidth]{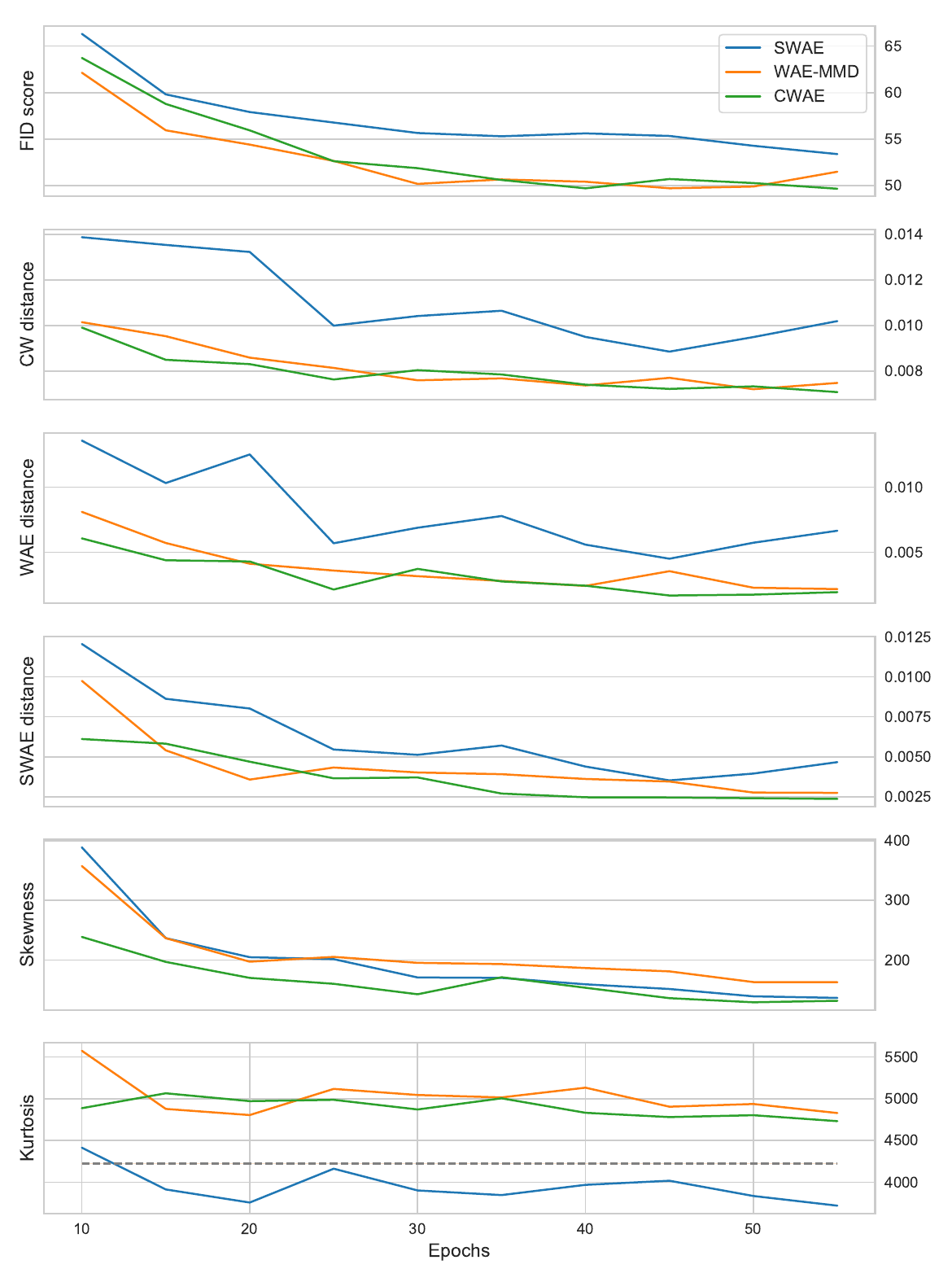}
\caption{CELEB~A trained CWAE, WAE, and SWAE models with FID score, kurtosis and skewness, as well as CW-, WAE-, and SWAE-distances on original WAE-MMD architecture from~\citet{tolstikhin2017wasserstein}. All values are mean from 5 models trained for each architecture. Confidence intervals represent the standard deviation. Optimum kurtosis is marked with a dash line.}
\label{fig:conv1}
\end{figure*}

\subsection{Qualitative tests}
\label{se:ex:qual}
The quality of a generative model is typically evaluated by examining generated samples or by interpolating between samples in the hidden space. We present such a~comparison between CWAE with WAE-MMD in Figure~\ref{fig:celeb}. We follow the same procedure as in~\citet{tolstikhin2017wasserstein}. In particular, we use the same base neural architecture for both CWAE and WAE-MMD. We consider for each model (i) interpolation between two random examples from the test set (leftmost in Figure~\ref{fig:celeb}), (ii) reconstruction of a~random example from the test set (middle column in Figure~\ref{fig:celeb}), and finally a~sample reconstructed from a~random point sampled from the prior distribution (right column in Figure~\ref{fig:celeb}). The experiment shows that there are no perceptual differences between CWAE and WAE-MMD generative distribution.

In the next experiment we qualitatively assess normality of the latent space. This will allow us to ensure that CWAE does not compromise on the normality of its latent distribution, which is a part of the cost function for all the models except AE. We compare CWAE\footnote{Since~\eref{eq:pi} is valid for dimensions $D\geq 20$, to implement CWAE in 2-dimensional latent space we apply equality $\1F1(1/2,1,-s)= e^{-\frac{s}{2}} I_0\left(\frac{s}{2}\right)$ jointly with the approximate formula~\citep[p.~378]{abramowitz1964handbook} for the Bessel function of the first kind $I_0$
.} with VAE, WAE and SWAE on the MNIST data with using 2-dimensional latent space and a~two dimensional Gaussian prior distribution. Results are reported in Figure~\ref{fig:latent2D}. As is readily visible, the latent distribution of CWAE is as close, or perhaps even closer, to the normal distribution than that of the other models. 

To summarize, both in terms of perceptual quality and satisfying normality objective, CWAE matches WAE-MMD. The next section will provide more quantitative studies.

\begin{figure*}[htb]
\normalsize
\begin{center}
\begin{tabular}{@{}c@{}c@{}c@{}c@{}c@{}}
\includegraphics[width=0.20\textwidth]{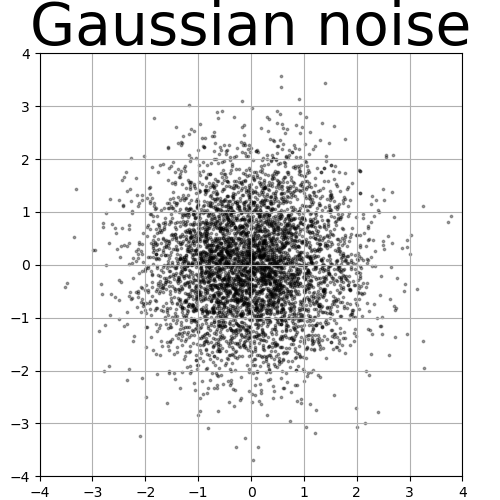}&
\includegraphics[width=0.20\textwidth]{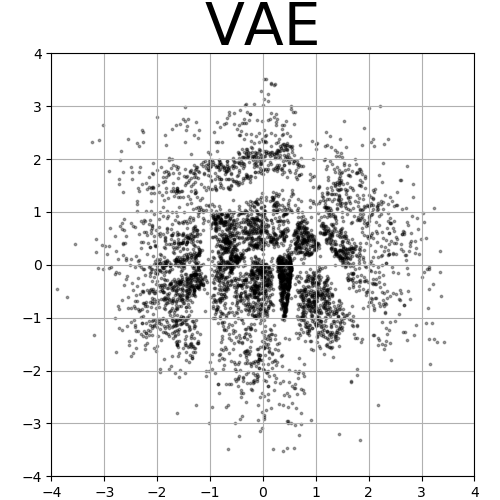}&
\includegraphics[width=0.20\textwidth]{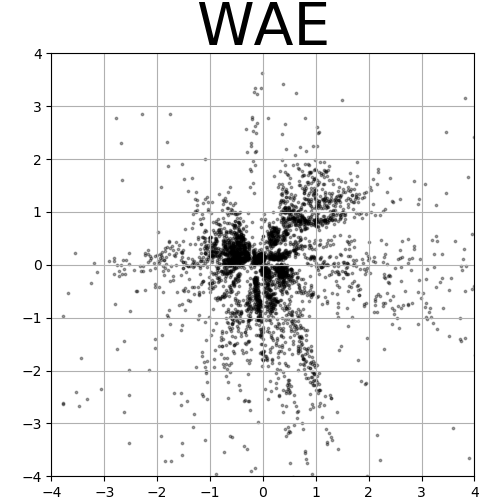}&
\includegraphics[width=0.20\textwidth]{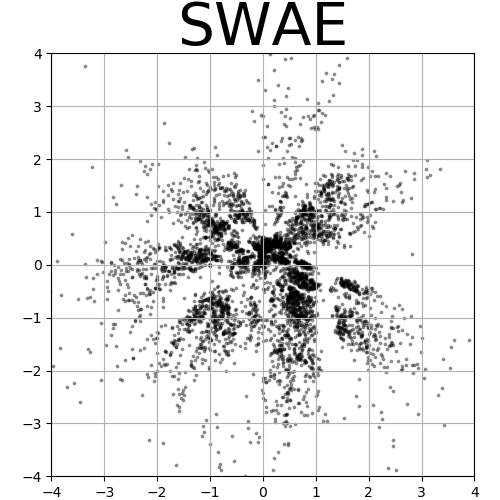}&
\includegraphics[width=0.20\textwidth]{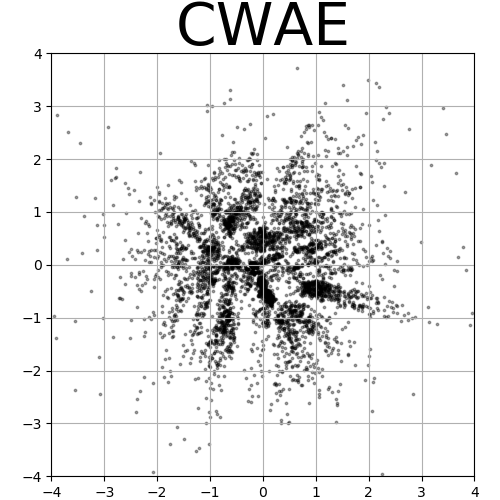}
\end{tabular}
\end{center}
\caption{Latent distribution of CWAE is close to the normal distribution. Each subfigure presents points sampled from two-dimensional latent spaces, VAE, WAE, SWAE, and CWAE (left to right). All trained on the MNIST data set.}
\label{fig:latent2D} 
\end{figure*}

\begin{table}[htb]
\normalsize
\caption{Comparison between classical cost function of WAE, SWAE and version with logarithm (method names with a -LOG suffix). }
\begin{center}
{\small
\begin{tabular}[width=\textwidth]{lllrrrrr}
\toprule
Data set & Method & Learning & $\lambda$ & Skewness & Kurtosis & Rec. & FID \\
 &  & rate & & & (normalized) & error & score\\
\midrule             
MNIST   & SWAE & 0.001 & 1.0 &  6.01 & 10.72 & 10.99 & 29.76 \\
        & SWAE-LOG & 0.0005 & 1.0 &  2.36 & 12.20 & 11.42 & 24.89 \\
        & WAE & 0.0005 & 1.0 &  11.70 & 8.34 & 11.14 & 27.65 \\
        & WAE-LOG & 0.001 & 1.0 &  18.22 & 61.04 & 13.17 & 36.08 \\        
        & CWAE & 0.001 & 1.0 &  12.21 & 35.88 & 11.25 &  {\bf 23.63} \\        
\midrule             
FASHION & SWAE & 0.001 & 1.0 &  8.22 & 11.58 & 9.91 & 74.32 \\
MNIST   & SWAE-LOG & 0.0005 & 1.0 &  1.36 & 13.64 & 9.83 & 61.72 \\
        & WAE & 0.001 & 1.0 &  22.17 & 37.96 & 9.90 & 69.16 \\
        & WAE-LOG & 0.005 & 1.0 &  53.37 & 66.01 & 16.14 & 99.51 \\
        & CWAE & 0.001 & 1.0 &  30.38 & 98.58 & 9.84 & {\bf  57.06} \\        
\midrule             
CIFAR10 & SWAE & 0.001 & 1.0 &  517.32 & 121.17 & 25.42 & 141.91 \\
         & SWAE-LOG & 0.0005 & 1.0 &  157.14 & 234.52 & 26.25 &  {\bf 119.89} \\
        & WAE & 0.001 & 1.0 &  1105.73 & 2097.14 & 25.04 & 129.37 \\
        & WAE-LOG & 0.001 & 1.0 &  $1.1\times10^8$ & $4.9\times10^5$ & 28.25 & 136.25 \\
        & CWAE & 0.001 & 1.0 &  176.60 & 1796.66 & 25.93 &  120.02 \\     
\midrule             
CELEB~A  & SWAE & 0.0001 & 100.0 &  141.17 & 222.02 & 85.97 & 53.85 \\
        & SWAE-LOG & 0.0005 & 10.0 &  132.54 & 465.39 & 85.82 & 53.46 \\
        & WAE & 0.0005 & 100.0 &  162.67 & 604.09 & 86.38 & 51.51 \\
        & WAE-LOG & 0.0001 & 1.0 &  514.43 & 2154.39 & 82.53 & 58.10 \\        
        & CWAE & 0.0005 & 5.0 &  130.08 & 542.42 & 86.89 & {\bf 49.69} \\        
\bottomrule
\end{tabular}
}
\end{center}
\label{tab:comp_log}
\end{table}

\subsection{Quantitative tests}
\label{se:ex:quant}

In order to quantitatively compare CWAE with other models, in the first experiment we follow the experimental setting and use the same architecture as in~\citet{tolstikhin2017wasserstein}. In particular, we use the Fr\'{e}chet Inception Distance (FID)~\citep{heusel2017gans}. 

In agreement with the qualitative studies, we observe FID  
of CWAE to be similar or slightly better than WAE-MMD. We highlight that CWAE on CELEB~A achieves $49.69$ FID score compared to $51.51$ and $53.85$ achieved by WAE-MMD and SWAE, respectively, see Figure~\ref{fig:conv1} and Table~\ref{tab:comp}.


Next, motivated by Remark~\ref{re:wazny} we propose a novel method for quantitative assessment of the models based on their comparison to standard normal distribution in the latent. To achieve this we have decided to use one of the most popular statistical normality tests, i.e. Mardia tests~\citep{henze2002invariant}. Mardia's normality tests are based on verifying whether the skewness $b_{1,D}(\cdot)$ and kurtosis $b_{2,D}(\cdot)$ of a sample $X=(x_i)_{i=1..n} \subset \R^D$:
\begin{linenomath*}
\begin{equation*}
b_{1,D}(X)=\tfrac{1}{n^2}
\sum_{j,k}(x_j^Tx_k)^3, \text{ and }
b_{2,D}(X)=\tfrac{1}{n}
\sum_{j}\|x_j\|^4
\end{equation*}
\end{linenomath*}
are close to that of standard normal density. The expected Mardia’s skewness and kurtosis for standard multivariate normal distribution is $0$ and $D(D + 2)$, respectively. To enable easier comparison in experiments we consider also the value of the normalized Mardia's kurtosis given by 
$b_{2,D}(X)-D(D+2)$,
which equals zero for the standard normal density. 

Results are presented in Figure~\ref{fig:conv1} and Table~\ref{tab:comp}. In Figure~\ref{fig:conv1} we report for CELEB~A data set the value of FID score, Mardia's skewness and kurtosis during learning process of WAE, SWAE and CWAE (measured on the validation data set). 

WAE, SWAE and CWAE models obtain the best reconstruction error, comparable to AE. VAE model exhibits a slightly worse reconstruction error, but values of kurtosis and skewness indicating their output is closer to normal distribution. As expected, the output of AE is far from  normal distribution; its kurtosis and skewness grow during learning. This arguably less standard evaluation, which we hope will find adoption in the community, serves as yet another evidence that \emph{CWAE has strong generative capabilities which at least match performance of WAE-MMD.} Moreover we observe that VAE model's output distribution is closest to the normal distribution, at the expense of the reconstruction error, which is reflected by the blurred reconstructions typically associated with VAE model.

At the end of these subsection we compare our method with classical approaches WAE-MMD and SWAE with modified cost function.
More precisely, similarly to CWAE we use logarithm in cost function in WAE and SWAE, see Table~\ref{tab:comp_log}. 

At it was mentioned, adding logarithm to WAE-MMD dos not work, since penalty used in WAE-MMD is not precisely the population MMD, but a sample based U-statistic. In consequence cost function can be negative from time to time. Thus the $\log$ version is not suitable for the WAE-MMD version. On the other hand, logarithm improves learning process in the case of CWAE and SWAE.

\subsection{Comparison of learning times}\label{ap:B}

We expect our closed-form formula to lead to a speedup in training time. Indeed, we found that for batch-sizes up to $1024$ CWAE is faster (in terms of time spent per batch) than other models. More precisely, CWAE is approximately $2\times$ faster up to $256$ batch-size.

Figure~\ref{fig:proc_comparison} gives comparison of mean learning time for different most frequently used batch-sizes. Time spent on processing a batch is actually smaller for CWAE for a practical range of batch-sizes $[32,512]$. For batch-sizes larger than $1024$, CWAE is slower due to its quadratic complexity with respect to the batch-size. However, we note that batch-sizes larger even than $512$ are relatively rarely used in practice for training autoencoders.

\begin{figure}[htb]
\centering
\includegraphics[width=0.85\linewidth]{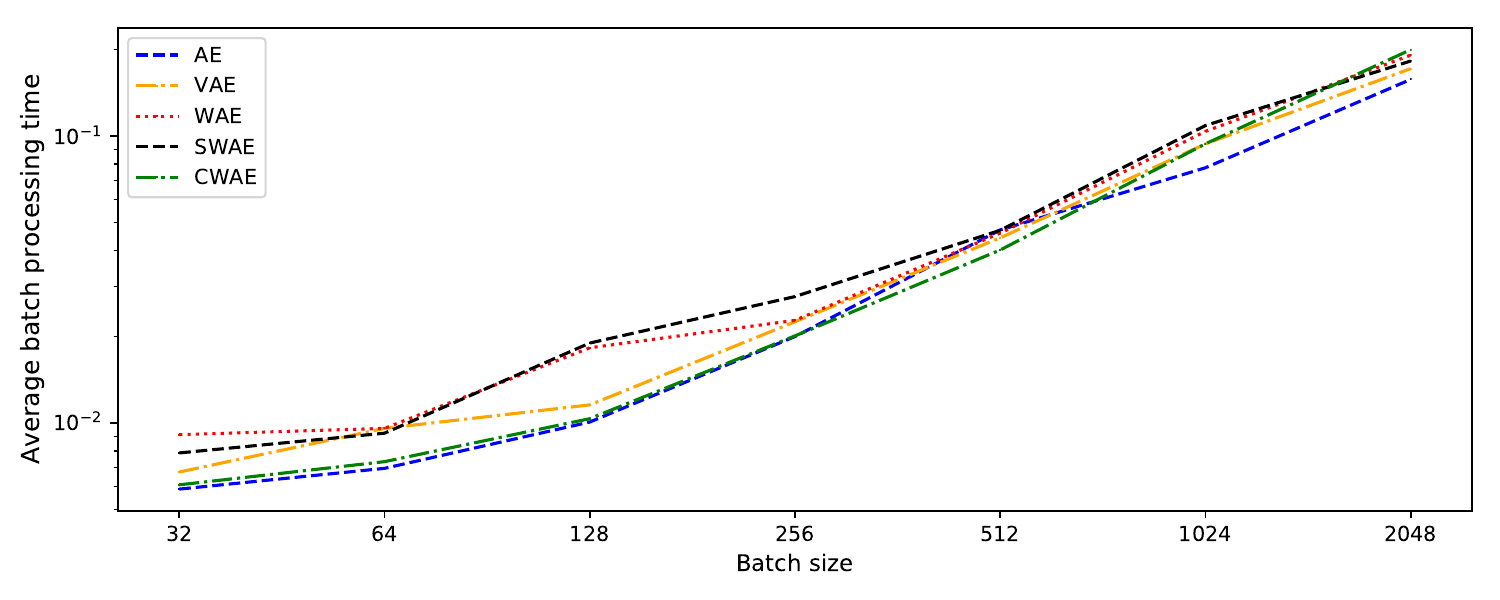}
\caption{Comparison of mean batch learning time (times are in log-scale) for different algorithms in seconds, all for the same architecture like the one in~\citet{tolstikhin2017wasserstein} and all requiring similar number of epochs to train the full model. The times may differ for computer architectures with more/less memory on a GPU card.
}
\label{fig:proc_comparison}
\end{figure}



\subsection{Architecture}\label{app:architectures}

At the end of this section we describe details of architecture used in our experiments.

MNIST/Fashion-MNIST ($28\times28$ images): an encoder-decoder feed-forward architecture:
\begin{itemize}[itemindent=-1.5em]
\item[] \textbf{encoder} three feed-forward ReLU layers, 200 neurons each
\item[]\textbf{latent} $8$-dimensional,
\item[]\textbf{decoder} three feed-forward ReLU layers, 200 neurons each.
\end{itemize}

CIFAR-10 dataset ($32\times$ images with three color layers): a convolution-deconvolution network
\begin{itemize}[leftmargin=*]
\item[]\textbf{encoder}
\begin{itemize}[leftmargin=*]
\item[] four convolution layers with $2\times2$ filters, the second one with $2\times2$ strides, other non-strided (3, 32, 32, and 32 channels) with ReLU activation,
\item[] 128 ReLU neurons dense layer,
\end{itemize}
\item[]\textbf{latent} $64$-dimensional,
\item[]\textbf{decoder}
\begin{itemize}[leftmargin=*]
\item[] two dense $ReLU$ layers with $128$ and $8192$ neurons,
\item[] two transposed-convolution layers with $2\times2$ filters (32 and 32 channels) and ReLU activation,
\item[] a transposed convolution layer with $3\times3$ filter and $2\times2$ strides (32 channels) and ReLU activation,
\item[] a transposed convolution layer with $2\times2$ filter (3 channels) and sigmoid activation.
\end{itemize}
\end{itemize}

CELEB~A (with images centered and cropped to $64\times64$ with 3 color layers): in order to have a direct comparison to WAE-MMD model on CELEB~A, an identical architecture was used as that in~\citet{tolstikhin2017wasserstein} utilized for the WAE-MMD model (WAE-GAN architecture is, naturally, different):
\begin{itemize}[leftmargin=*]
\item[]\textbf{encoder}
\begin{itemize}[leftmargin=*]
\item[] four convolution layers with $5\times5$ filters, each layer followed by a batch normalization (consecutively 128, 256, 512, and 1024 channels) and ReLU activation,
\end{itemize}
\item[]\textbf{latent} 64-dimensional,
\item[]\textbf{decoder} 
\begin{itemize}[leftmargin=*]
\item[] dense 1024 neuron layer,
\item[] three transposed-convolution layers with $5\times5$ filters, and each layer followed by a batch normalization with ReLU activation (consecutively 512, 256, and 128 channels),
\item[] transposed-convolution layer with $5\times5$ filter and 3 channels, clipped output value.
\end{itemize}
\end{itemize}

The last layer returns the reconstructed image. The results for all above architectures are given in Table~\ref{tab:comp}. All networks were trained with the Adam optimizer~\citep{kingma2014adam}. The hyper-parameters used were $learning\; rate=0.001$, $\beta_1=0.9$, $\beta_2=0.999$, $\epsilon=1e-8$. MNIST and CIFAR 10 models were trained for 500 epochs, CELEB~A for 55.

Similarly to~\citet{tolstikhin2017wasserstein}, models were trained using Adam with for 55 epochs, with the same optimizer parameters.

\section{Conclusions}\label{se:con}



In the paper we have presented a new autoencoder based generative model CWAE, which matches and in some cases improves results of WAE-MMD, while using \emph{a cost function given by a simple closed analytic formula.} We hope this result will encourage future work in developing simpler to optimize analogs of strong neural models. 

Crucial in the construction of CWAE is the use of the developed Cramer-Wold metric between samples and distributions, which can be effectively computed for Gaussian mixtures. As a consequence, we obtain a reliable measure of the divergence from normality. Future work could explore use of the Cramer-Wold distance in other settings, in particular in adversarial models.

\section{Acknowledgements}

The work of P. Spurek was supported by the National Centre of Science (Poland) Grant No. 2015/19/D/ST6/01472. The work of J. Tabor was supported by the National Centre of Science (Poland) Grant No. 2015/19/B/ST6/01819. The work of I. Podolak was supported by the National Centre of Science (Poland) Grant No. 2017/25/B/ST6/01271.


\vskip 0.2in
\bibliography{ref}

\end{document}